\documentclass{article}

\usepackage[preprint]{neurips_2026}

\usepackage[utf8]{inputenc}
\usepackage[T1]{fontenc}
\usepackage{hyperref}
\usepackage{url}
\usepackage{booktabs}
\usepackage{amsfonts}
\usepackage{nicefrac}
\usepackage{microtype}
\usepackage{xcolor}
\usepackage{amsmath, amssymb}
\usepackage{graphicx}

\usepackage{multirow}
\usepackage{algorithm}
\usepackage{algpseudocode}
\usepackage{amsthm}
\newtheorem{lemma}{Lemma}
\newtheorem{theorem}{Theorem}
\newtheorem{corollary}{Corollary}
\usepackage{array}
\newcolumntype{C}[1]{>{\centering\arraybackslash}p{#1}}

\title{Factorization-Error-Free Discrete Diffusion Language Model via Speculative Decoding}

\author{%
  Xun Fang \\
  East China Normal University \\
  \texttt{51264404020@stu.ecnu.edu.cn} \\
  \And
  Yunchen Li \\
  East China Normal University \\
  \texttt{52284404001@stu.ecnu.edu.cn} \\
  \And
  Hang Yuan \\
  East China Normal University \\
  Beijing Zhongguancun Academy \\
  \texttt{52274404018@stu.ecnu.edu.cn} \\
  \And
  Zhou Yu\thanks{Corresponding author.} \\
  East China Normal University \\
  \texttt{zyu@stat.ecnu.edu.cn}
}

\begin{document}
\renewcommand{\thefootnote}{}
\footnotetext{The authors contribute equally to the paper and are listed in alphabetical order.}
\maketitle

\begin{abstract}
Discrete diffusion language models improve generation efficiency through parallel token prediction, but standard $X_0$ prediction methods introduce factorization errors by approximating the clean token posterior with independent token-wise distributions. This paper proposes Factorization-Error-Free Discrete Diffusion Language Modeling (FeF-DLLM), which replaces independent clean-token prediction with an exact prefix-conditioned factorization of the clean posterior to better preserve token dependencies. To reduce the sequential cost introduced by prefix conditioning, FeF-DLLM further incorporates speculative decoding within diffusion denoising, accelerating inference while maintaining the parallel prediction and re-masking properties of DLLMs. Theoretically, we prove that FeF-DLLM generates from the true joint distribution and derive its expected acceleration ratio. Experiments on GSM8K, MATH, HumanEval, and MBPP demonstrate that our method improves accuracy by an average of 5.04 percentage points while achieving an average inference speedup of $3.86\times$.
\end{abstract}

\section{Introduction}

Diffusion models \citep{ho2020denoising,lipmanflow,songscore,songdenoising} have become one of the most successful classes of generative models in recent years, achieving strong performance in image generation \citep{rombach2022high,peebles2023scalable,ma2024sit}, video generation \citep{ho2022video,bar2024lumiere}, and many other domains \citep{wu2024seesr,yim2024diffusion,wang2025diffusion}. 
Recently, several works have extended diffusion modeling to discrete spaces~\citep{austin2021structured,lou2024discrete,gat2024discrete,sahoo2024simple} and applied discrete diffusion language models (DLLMs) to large language modeling~\citep{nie2025llada,ye2025dream,bie2025llada20scalingdiffusionlanguage}. 
These models have shown competitive performance with autoregressive models while enabling a different generation paradigm based on iterative denoising and parallel token prediction.

During generation, DLLMs usually predict multiple tokens in parallel and combine these token-wise predictions to form the final output. 
Although this design improves generation efficiency, it introduces a factorization error because the joint distribution over clean tokens is approximated by a product of independent token distributions. 
Recent works have attempted to address this issue from different perspectives. 
ReDi~\citep{yooredi} mitigates factorization error through a rectified-flow formulation and newly constructed paired training data, but it does not fully remove the error. 
Other methods propose improved sampling procedures~\citep{liudiscrete,lavenant2025error}, often at the cost of significantly slower generation. 
Some approaches integrate DLLMs with speculative decoding~\citep{campbell2025self,cheng2025deerdraftdiffusionverify}; however, such designs introduce autoregressive decoding behavior into DLLM inference, thereby undermining the distinctive non-autoregressive reasoning characteristics of DLLMs.

In this paper, we propose \emph{Factorization-Error-Free Discrete Diffusion Language Modeling} (FeF-DLLM). 
Built upon the classical $X_0$-prediction framework of DLLMs, FeF-DLLM analyzes the exact decomposition of the clean posterior distribution and derives a factorization-error-free generation objective. 
Instead of independently predicting each clean token, our method uses prefix-conditioned prediction to preserve dependencies among clean tokens. 
To improve inference efficiency, we further incorporate speculative decoding, which amortizes the sequential dependency introduced by prefix conditioning and accelerates generation while retaining the parallel prediction and re-masking properties of DLLMs.

We theoretically show that FeF-DLLM can generate samples from the true conditional joint distribution when the position-conditioned target model is well specified, and we analyze the expected speedup brought by speculative decoding. We evaluate FeF-DLLM on standard mathematical reasoning and code generation benchmarks, including GSM8K, MATH, HumanEval, and MBPP. Experimental results show that FeF-DLLM improves accuracy by an average of 5.04 percentage points and achieves an average inference speedup of $3.86\times$, demonstrating its effectiveness in improving generation quality while substantially accelerating inference.

Our contributions are summarized as follows:
\begin{itemize}
    \item We analyze the factorization error in existing DLLMs and show that it can be eliminated through an exact prefix-conditioned factorization of the clean posterior.

    \item We use speculative decoding to accelerate the resulting prefix-conditioned inference procedure and derive its expected speedup.

    \item We conduct extensive experiments to evaluate FeF-DLLM. Experimental results demonstrate that FeF-DLLM consistently improves generation quality over the baseline while simultaneously achieving substantial wall-clock acceleration.
\end{itemize}

\section{Preliminary}
\subsection{Discrete Diffusion Language Models}

Discrete diffusion language models define a diffusion process directly over discrete tokens. Building on the D3PM framework \citep{austin2021structured}, we define the forward diffusion process as a fixed Markov chain $q(X_{1:T}\mid X_0)=\prod_{t=1}^{T}q(X_t\mid X_{t-1})$, in which each transition kernel is parameterized by a categorical transition matrix \(Q_t\in\mathbb{R}^{S\times S}\), where \(S\) denotes the vocabulary size. For a one-hot token \(X_{t-1}\), the one-step corruption process is
\begin{equation*}
q(X_t\mid X_{t-1})=\mathrm{Cat}(X_t;\,p=X_{t-1}Q_t),
\end{equation*}

where $\mathrm{Cat}(x; p)$ denotes a categorical distribution over the one-hot row vector $x$, with probabilities given by row vector $p$. Let \(\bar Q_t=Q_1Q_2\cdots Q_t\). The marginal at timestep \(t\) is \(q(X_t\mid X_0)=\mathrm{Cat}(X_t;\,p=X_0\bar Q_t)\), enabling noisy tokens to be sampled directly from the clean input. For sequence data, corruption is typically applied independently across token positions. Common choices of \(Q_t\) include uniform transitions, nearest-neighbor transitions in embedding space, and absorbing-state transitions that gradually replace tokens with a special \([\mathrm{MASK}]\) token.

The forward posterior has the following closed form:
\begin{equation*}
q(X_{t-1}\mid X_t,X_0)=\mathrm{Cat}\left(X_{t-1};p=\frac{X_tQ_t^\top \odot X_0\bar Q_{t-1}}{X_0\bar Q_t X_t^\top}\right).
\end{equation*}
The generative process is a learned reverse Markov chain \(p_\theta(X_{0:T})=p(X_T)\prod_{t=1}^{T}p_\theta(X_{t-1}\mid X_t)\). In the \(X_0\)-prediction parameterization, the model first predicts the clean token from the corrupted input, denoted as \(\hat{X}_0=f_\theta(X_t,t)\). Then the reverse transition is constructed by plugging this prediction into the analytic forward posterior:
\begin{equation}
p_\theta(X_{t-1}\mid X_t) = q(X_{t-1}\mid X_t,\hat{X}_0).
\label{eq:forward}
\end{equation}

Equivalently, the model uses the predicted clean token to determine how to denoise \(X_t\) by one step. The corresponding clean-token prediction objective is
\begin{equation}
\mathcal{L}=\mathbb{E}_{q({X}_0)}\mathbb{E}_{t}\mathbb{E}_{q(X_t\mid X_0)}\left[-\sum_{i=1}^{N}\log p_\theta(X_0^{i}\mid X_t,t)\right].
\label{eq:loss function}
\end{equation}

\subsection{Speculative Decoding}

Speculative decoding \citep{leviathan2023fast,chen2023accelerating} accelerates autoregressive inference by using a smaller and faster approximation model to propose multiple candidate tokens, which are then verified in parallel by the target model. Let $M_{\pi}$ denote the target model, and let $ \pi(X^i \mid X^{<i})$ be its next-token distribution given the prefix $X^{<i}$. We further introduce an efficient approximation model $M_{\rho}$, whose next-token distribution is denoted by $\rho(X_t \mid X^{<i})$. Speculative decoding proceeds as follows. First, the approximation model $M_{\rho}$ autoregressively generates $\gamma$ candidate tokens, denoted by $X^1, \ldots, X^{\gamma}$. Then, the target model $M_{\pi}$ evaluates the corresponding candidate prefixes in parallel and obtains the target distributions at each position. For the $i$-th proposed token $X^i$, the algorithm accepts it with probability
\[
    \min\left(1, \frac{\pi_i(X^i)}{\rho_i(X^i)}\right),
\]
where $\pi_i$ and $\rho_i$ denote the target and approximation distributions under the corresponding prefix, respectively. If a proposed token is rejected, the algorithm resamples from the corrected distribution
\[
    \pi'(X)=\operatorname{norm}\left(\max\left(0, \pi(X) - \rho(X)\right)\right).
\]
This correction ensures that the resulting sample is still exactly distributed according to the target distribution $\pi$. Therefore, speculative decoding produces samples from the same distribution as direct decoding from $M_{\pi}$. Benefiting from the parallel verification of multiple draft tokens, speculative decoding can substantially improve generation speed.

\subsection{Setting and Notation}

We consider a \(d\)-dimensional discrete random variable \(X=(X^1,\dots,X^d)^\top\) defined on \(\mathcal{V}^d\). Subscripts are used to index decoding steps in diffusion language model inference, while superscripts are used to index token positions in a sequence. For example, $X_t^i$ denotes the random token at position $i$ when the diffusion decoder is at step $t$. We write $X^{<i}$ and $X^{>i}$ to denote the subsequence of tokens before and after position $i$. In the context of speculative decoding, we use $\pi$ to denote the target distribution and $\rho$ to denote the draft distribution. \(\odot\) is denoted as element-wise multiplication. 


\section{Methodology}
\label{sec:methodology}
\subsection{Analysis of Factorization Error}

We first revisit the commonly used $X_0$-prediction parameterization in discrete diffusion language models. 
Given a corrupted sequence $X_t$, the reverse transition in Eq.~\ref{eq:forward} is constructed from an estimate of the clean-data posterior $p(X_0\mid X_t)$. 
In principle, this posterior is a distribution over the full discrete sequence space $\mathcal{V}^d$, whose cardinality grows exponentially with the sequence length $d$. 
Directly modeling such a joint distribution is computationally intractable for neural language models. 
Therefore, existing DLLMs usually adopt a token-wise prediction objective, in which the model predicts each clean token independently conditioned on the same corrupted input:
\begin{equation}
    p(X_0 \mid X_t) 
    \approx 
    \prod_{i=1}^d p(X_0^i \mid X_t).
    \label{eq:mean-field-factorization}
\end{equation}

However, Eq.~\ref{eq:mean-field-factorization} relies on an independence assumption across dimensions. Such an assumption is generally incompatible with natural language, where tokens exhibit strong syntactic and semantic dependencies. The exact decomposition of the clean posterior follows the chain rule:
\begin{equation*}
    p(X_0 \mid X_t)
    =
    \prod_{i=1}^d
    p(X_0^i \mid X_t, X_0^{<i}).
    \label{eq:autoregressive-factorization}
\end{equation*}
Compared with Eq.~\ref{eq:mean-field-factorization}, this factorization preserves dependencies among clean tokens by conditioning each position on the previously recovered clean prefix $X_0^{<i}$. 
The discrepancy between the token-wise approximation and the exact factorization is therefore characterized by the missing dependence on $X_0^{<i}$. To make this discrepancy explicit, we derive the following identity.

\begin{lemma}
\label{lemma:predict}
Assume that the forward corruption process factorizes across token positions, 
i.e., $q(X_t\mid X_0)=\prod_{j=1}^{d}q(X_t^j\mid X_0^j)$. 
For any position $i$, the autoregressive clean-token posterior satisfies
\begin{equation*}
\begin{split}
    p(X_0^i \mid X_t, X_0^{<i})
    &=p(X_0^i \mid X_t^{\ge i}, X_0^{<i})\\
    &\propto
    p(X_0^i \mid X_t)
    \frac{p(X_0^i \mid X_t^{>i}, X_0^{<i})}
    {p(X_0^i \mid X_t^{-i})},
\end{split}
\label{eq:conditional-identity}
\end{equation*}
where $X_t^{-i}=(X_t^{<i},X_t^{>i})$, and the proportionality is over $X_0^i$.
\end{lemma}

Lemma~\ref{lemma:predict} shows that the desired autoregressive posterior $p(X_0^i\mid X_t,X_0^{<i})$ is equivalent to conditioning on
$(X_t^{\ge i},X_0^{<i})$, and differs from the independent predictor
$p(X_0^i\mid X_t)$ by the normalized correction term
$p(X_0^i\mid X_t^{>i},X_0^{<i})/p(X_0^i\mid X_t^{-i})$.
This term captures the additional dependence on the clean prefix and the
remaining corrupted context. Therefore, predicting the $i$-th clean token should
account not only for the corrupted sequence $X_t$, but also for previously
determined clean tokens; ignoring this dependence leads to factorization error in
the learned reverse process.

Motivated by this observation, we replace the independent clean-token predictor with a prefix-conditioned predictor. 
Instead of estimating $p_\theta(X_0^i\mid X_t,t)$ independently for each position, we train the model to approximate
\[
    p_\theta(X_0^i\mid X_t^{\ge i}, X_0^{<i}, t).
\]
 
This parameterization preserves the left-to-right dependency structure of the clean sequence while still leveraging the bidirectional corrupted context available in diffusion inference. The resulting training objective is
\begin{equation}
    \mathcal{L}
    =
    \mathbb{E}_{q(X_0)}
    \mathbb{E}_{t}
    \mathbb{E}_{q(X_t\mid X_0)}
    \left[
        -\sum_{i=1}^{d}
        \log p_\theta(X_0^{i}\mid X_t^{\ge i}, X_0^{<i}, t)
    \right].
    \label{eq:modified loss function}
\end{equation}
This objective has the same supervised clean-token prediction form as Eq.~\ref{eq:loss function}. Consequently, it can be optimized by finetuning existing DLLM backbones with modified input, without changing the underlying discrete diffusion forward process. The corresponding training procedure is summarized in Algorithm~\ref{alg:position_conditioned_training} in Appendix~\ref{sec:supp_algorithm}.

\subsection{Acceleration via Speculative Decoding}
\label{sec:speculative-diffusion}

The prefix-conditioned predictor introduced in Eq.~\ref{eq:modified loss function} enables inference from the correct joint distribution of clean tokens, thereby eliminating the factorization error induced by the independence assumption, but it also changes the inference pattern of DLLMs. 
In standard DLLM inference, all clean-token predictions can be produced in parallel from the corrupted sequence $X_t$. 
In contrast, $p_\theta(X_0^i\mid X_t^{\ge i}, \hat X_0^{<i}, t)$ depends on the previously recovered clean prefix $\hat X_0^{<i}$. 
Therefore, directly sampling from this distribution requires left-to-right decoding within each denoising step, which may reduce the parallel efficiency of diffusion-based generation.

To mitigate this sequential bottleneck, we adopt speculative decoding within each diffusion denoising step. 
The key idea is to use a fast draft model to propose multiple clean tokens in parallel, and then verify these proposals with the prefix-conditioned target model. The complete inference algorithm is provided in Algorithm~\ref{alg:fef_dllm_full_inference} in Appendix~\ref{sec:supp_algorithm}.
Suppose that the first $m$ positions have already been determined, denoted by $\hat X_0^{<m}$. 
We consider a speculative window of length $k$, covering positions $m+1,\ldots,m+k$. The draft model defines a distribution
\begin{equation*}
    \rho_\phi^i(X_0^i)
    :=
    \rho_\phi(X_0^i \mid \hat X_0^{<m}, X_t^{>m}, t),
    \qquad i=m+1,\ldots,m+k.
\end{equation*}
Unlike the target model, the draft model conditions only on the already verified prefix $\hat X_0^{<m}$ and the remaining corrupted context $X_t^{>m}$. 
Thus, the draft distributions for all positions in the speculative window can be computed in parallel. 
We then sample draft tokens $\tilde X_0^{m+1},\ldots,\tilde X_0^{m+k}\sim\prod_{i=m+1}^{m+k}\rho_\phi^i(\cdot).$ After the draft tokens are generated, the target model verifies them from left to right. For each position $i=m+1,\ldots,m+k$, the target distribution is evaluated using the prefix formed by the already accepted tokens and the preceding draft tokens:
\begin{equation*}
    \pi_\theta^i(X_0^i):=p_\theta(X_0^i \mid X_t^{\ge i}, \tilde X_0^{<i}, t).
\end{equation*}
The proposed token $\tilde X_0^i$ is accepted with probability
\begin{equation*}
    \min\left(1,\frac{\pi_\theta^i(\tilde X_0^i)}{\rho_\phi^i(\tilde X_0^i)}\right).
\end{equation*}

If $\tilde X_0^i$ is accepted, we set $\hat X_0^i=\tilde X_0^i$ and continue to verify the next position. 
If it is rejected, we resample $\hat X_0^i$ from the corrected distribution
\begin{equation*}
    \pi_\theta^{\prime i}(x)
    =
    \operatorname{norm}
    \left(
        \max(0,\pi_\theta^i(x)-\rho_\phi^i(x))
    \right),
    \label{eq:speculative-correction}
\end{equation*}
and terminate the current speculative window. 
The next speculative step then starts from the updated verified prefix. 
This procedure is repeated until all positions in the block have been inferred, yielding a complete prediction $\hat X_0$. Once $\hat X_0$ is obtained, we construct the reverse transition following the standard $X_0$-prediction parameterization:
\begin{equation*}
    p_\theta(X_{t-1}\mid X_t)=q(X_{t-1}\mid X_t,\hat X_0),
\end{equation*}
as in Eq.~\ref{eq:forward}. 
Therefore, our method modifies only the clean-token prediction stage, while leaving the discrete forward process and the analytic posterior transition unchanged.

The choice of the draft model is flexible. 
In principle, any model that approximates the target conditional distribution and supports efficient sampling can be used as $\rho_\phi$. 
In practice, we use a DLLM-style draft model because its token predictions can be computed in parallel within a speculative window, resulting in an $O(1)$ drafting cost with respect to the window length. 
Moreover, using a draft model with the same architecture family as the target model often yields higher acceptance rates, since the draft distribution is better aligned with the prefix-conditioned target distribution. The complete inference algorithm is provided in Algorithm~\ref{alg:fef_dllm_full_inference} in Appendix~\ref{sec:supp_algorithm}. The overall training and inference pipeline of FeF-DLLM is illustrated in Figure~\ref{fig:algorithm}.

\begin{figure}[t]
    \centering
    \includegraphics[width=\linewidth]{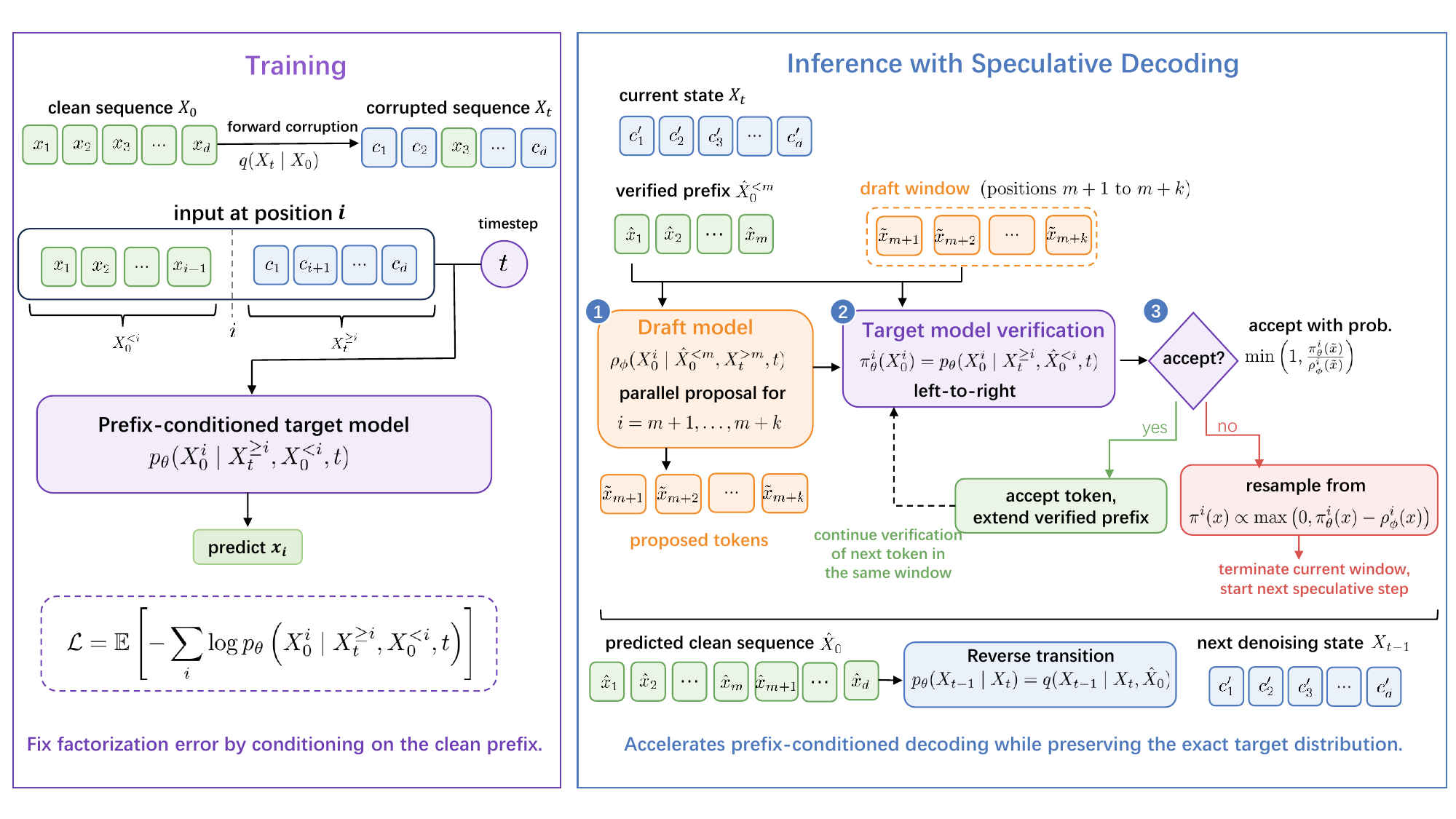}
    \caption{Overview of FeF-DLLM. Prefix-conditioned training predicts each clean token from the clean prefix $X_0^{<i}$ and the remaining corrupted context $X_t^{\ge i}$. During inference, a draft model proposes a speculative window in parallel, and a prefix-conditioned target model verifies the proposals left to right before constructing the standard reverse transition.}
    \label{fig:algorithm}
\end{figure}

\subsection{Theoretical Properties}

We establish two theoretical properties of FeF-DLLM. The first result shows that the proposed prefix-conditioned generation rule is distributionally exact whenever the target model matches the true position-conditioned posterior. The second result characterizes the expected progress of speculative verification under a standard independent-acceptance approximation.

\paragraph{Distributional exactness.}
Fix a denoising state $X_t$ and let
\[
J_t=\{j_1,\ldots,j_{L_t}\}, \qquad j_1<\cdots<j_{L_t},
\]
be the ordered set of positions to be updated at timestep $t$. Let $p^\star$ denote the oracle data law. We assume the target law $\pi_\theta$ induced by $p_\theta$ matches the true position-conditioned posterior under $p^\star$. For every $m\in\{1,\ldots,L_t\}$, we take the target model to be the oracle next-position law
\begin{equation}
\pi_\theta^{j_m}(x)
=
p^\star\!\left(
X_0^{j_m}=x
\,\middle|\,
X_t^{\ge j_m}, X_0^{<j_m}, t
\right).
\label{eq:oracle target}
\end{equation}

\begin{theorem}[Exact joint law]
Under Eq.~\ref{eq:oracle target}, the sequence produced by FeF-DLLM satisfies, for any $x_{J_t}\in\mathcal{V}^{L_t}$,
\[
\Pr\!\left(
\hat X_{0,J_t}=x_{J_t}
\,\middle|\,
X_t,t
\right)
=
\prod_{m=1}^{L_t}
p^\star\!\left(
X_0^{j_m}=x^{j_m}
\,\middle|\,
X_t^{\ge j_m},x_0^{<j_m},t
\right).
\]
Equivalently,
\[
\Pr\!\left(
\hat X_{0,J_t}=x_{J_t}
\,\middle|\,
X_t,t
\right)
=
p^\star(x_{J_t}\mid X_t,t).
\]
\label{theorem:distribution}
\end{theorem}

\begin{corollary}[Conditional correctness under resampling]
Consider any resampling pass $s\ge 1$, and let
\[
R^{(s)}=\{r_1^{(s)},\ldots,r_{L_s}^{(s)}\},\qquad
r_1^{(s)}<\cdots<r_{L_s}^{(s)},
\]
be the positions selected for resampling. Conditional on the previous sequence $\hat X_0^{(s-1)}$ and the selected set $R^{(s)}$, assume Eq.~\ref{eq:oracle target} holds on $R^{(s)}$, then for any $x_{R^{(s)}}\in\mathcal{V}^{L_s}$,
\[
\Pr\!\left(
\hat X_{0,R^{(s)}}^{(s)}=x_{R^{(s)}}
\,\middle|\,
\hat X_0^{(s-1)},R^{(s)},t
\right)
=
p^\star\!\left(
x_{R^{(s)}}
\,\middle|\,
\hat X_0^{(s-1)},R^{(s)},t
\right).
\]
\label{corollary:resample}
\end{corollary}

The corollary is a conditional statement: once the resampling set is fixed, the same distributional correctness argument applies to the selected positions. Therefore, repeated low-confidence resampling preserves the correctness of each conditional regeneration step.

\paragraph{Expected speculative progress.}
We next quantify the expected number of positions committed in one speculative round. Let $k$ denote the window size and let $\alpha$ denote the average probability that a draft token is accepted by the target model. Following the standard analysis of speculative decoding \citep{leviathan2023fast}, we assume that acceptance events within a window are independent.

\begin{theorem}[Expected committed length]
Let $C_k$ be the number of positions committed in one speculative round, where the first rejected position, if any, is also committed after correction. Then
\[
\mathbb{E}[C_k]
=
\sum_{\ell=0}^{k-1}\alpha^\ell
=
\begin{cases}
\dfrac{1-\alpha^k}{1-\alpha}, & 0\le \alpha <1,\\[6pt]
k, & \alpha=1.
\end{cases}
\]
\label{theorem:committed length}
\end{theorem}

\begin{corollary}[Idealized acceleration ratio]
Let $c_\rho$ and $c_\pi$ denote the wall-clock costs of one draft proposal pass and one target verification pass, respectively. Relative to a prefix-conditioned sequential target baseline that commits one position per target pass, the idealized expected acceleration ratio is
\[
S
=
\frac{\mathbb{E}[C_k]\,c_\pi}{c_\rho+c_\pi}.
\]
In particular, if $c_\rho \approx c_\pi$, then
\[
S
\approx
\frac{\mathbb{E}[C_k]}{2}
=
\begin{cases}
\dfrac{1-\alpha^k}{2(1-\alpha)}, & 0\le \alpha <1,\\[6pt]
\dfrac{k}{2}, & \alpha=1.
\end{cases}
\]
\label{corollary:acceleration}
\end{corollary}
Thus, larger windows and higher acceptance rates increase the expected number of committed positions per speculative round, while the actual wall-clock gain is additionally modulated by the relative costs of drafting and verification.

\section{Experiment}
\label{sec:experiments}

\subsection{Experimental Setup}
\label{sec:exp_setup}
We use a supervised finetuning set of 72 thousand prompt--response pairs for training, covering both code and math data. Each training example is organized as a prompt--response pair, where only response tokens are masked and predicted. The finetuned base model is optimized with the proposed position-conditioned objective in Eq.~\ref{eq:modified loss function}. We optimize all models with AdamW using a weight decay of 0.1. The learning rate is initialized at $1\times10^{-5}$, with 50 warmup steps, followed by a constant phase and a linear decay over the final 10\% of training steps to 0.1 times the peak learning rate. We use a per-device batch size of 1 and gradient accumulation over 2 steps, yielding a global batch size of 16. We train for 1 epochs and use bf16 mixed precision. 

During speculation, both the draft and target models in speculative decoding are instantiated with the fine-tuned model, and the window size is set to 16. This choice is further validated in the ablation study. All training and inference experiments are conducted on 8 NVIDIA A100 80GB GPUs.

\subsection{Main Results}
We evaluate the models on four widely used benchmarks: GSM8K \citep{cobbe2021gsm8k}, MATH \citep{hendrycks2021math}, HumanEval \citep{chen2021humaneval}, and MBPP \citep{austin2021mbpp}, covering both mathematical reasoning and code generation. To ensure reproducibility, we rely on the standardized OpenCompass \citep{2023opencompass} evaluation pipeline. We report two main metrics: Accuracy, which measures final task performance, and Speedup, which measures the wall-clock acceleration ratio relative to the baseline decoding method. In all experiments, we use the inference time of LLaDA-Instruct as the reference runtime, denoted as $1\times$.

We adopt LLaDA-Instruct~\citep{nie2025llada} as the backbone model and build our method upon it; consequently, LLaDA-Instruct serves as the primary baseline for comparison. In addition, we compare our approach with several representative methods, including SSD~\citep{gao2025self}, DDOSP~\citep{lavenant2025error}, and DCD~\citep{liudiscrete}. Since the original papers, with the exception of LLaDA-Instruct, do not report results on the benchmarks considered in this work, we reproduce all remaining baselines under a unified evaluation protocol. Further implementation details are provided in Appendix~\ref{sec:compared_methods}. It is worth noting that the original LLaDA decodes only one token at each step, which partially mitigates factorization error. To more directly evaluate the effectiveness of our method under a comparable multi-token decoding setting, we additionally reproduce a variant of LLaDA that decodes two tokens per step, denoted as LLaDA/2. The results are reported in Table~\ref{tab:main_results}.

\begin{table*}[t]
    \centering
    \small
    \setlength{\tabcolsep}{3.8pt}
    \caption{Main results on mathematical reasoning and code generation benchmarks. Best accuracy values are in bold, and fastest speed values are underlined. Mean denotes the average results across all four benchmarks. For FeF-DLLM, step denotes the number of diffusion model steps.}
    \label{tab:main_results}
    \resizebox{\textwidth}{!}{
    \begin{tabular}{@{}l cc cc cc cc cc@{}}
        \toprule
        \multirow{2}{*}{Methods}
        & \multicolumn{2}{c}{GSM8K}
        & \multicolumn{2}{c}{MATH}
        & \multicolumn{2}{c}{HumanEval}
        & \multicolumn{2}{c}{MBPP}
        & \multicolumn{2}{c}{Mean} \\
        \cmidrule(lr){2-3}
        \cmidrule(lr){4-5}
        \cmidrule(lr){6-7}
        \cmidrule(lr){8-9}
        \cmidrule(lr){10-11}
        & Acc. & Speed & Acc. & Speed & Acc. & Speed & Acc. & Speed & Acc. & Speed \\
        \midrule
        LLaDA \citep{nie2025llada}
        & 78.60 & $1.00\times$ & 26.60 & $1.00\times$ & 47.60 & $1.00\times$ & 34.20\footnotemark & $1.00\times$ & 46.75 & $1.00\times$ \\
        LLaDA/2 \citep{nie2025llada}
        & 76.42 & $1.92\times$ & 25.46 & $1.99\times$ & 32.32 & $2.02\times$ & 35.00 & $1.98\times$ & 42.30 & $1.98\times$ \\
        \midrule
        SSD \citep{gao2025self}
        & 77.10 & $2.23\times$ & 34.94 & $2.16\times$ & 43.09 & $2.12\times$ & 39.20 & $1.83\times$ & 48.58 & $2.09\times$ \\
        DDOSP \citep{lavenant2025error}
        & 74.15 & $1.92\times$ & 25.78 & $2.03\times$ & 28.05 & $2.01\times$ & 29.20 & $1.98\times$ & 39.30 & $1.98\times$ \\
        DCD \citep{liudiscrete}
        & 78.24 & $0.36\times$ & 26.36 & $0.51\times$ & \textbf{50.00} & $0.43\times$ & 37.60 & $0.40\times$ & 48.05 & $0.43\times$ \\
        \midrule
        FeF-DLLM (step=2)
        & 79.38 & $\underline{3.55\times}$ & 36.40 & $\underline{3.25\times}$ & 48.78 & $\underline{3.76\times}$ & \textbf{42.60} & $\underline{4.89\times}$ & 51.79 & $\underline{3.86\times}$ \\
        FeF-DLLM (step=4)
        & \textbf{79.68} & $2.14\times$ & \textbf{36.56} & $1.99\times$ & 49.39 & $2.21\times$ & \textbf{42.60} & $2.99\times$ & \textbf{52.06} & $2.33\times$ \\
        \bottomrule
    \end{tabular}
    }
\end{table*}
\footnotetext{The reproduced LLaDA accuracy on MBPP is 37.40; we report 34.20 from the original LLaDA result for consistency.}

We draw two key observations from Table~\ref{tab:main_results}. First, increasing the number of tokens decoded at each step can substantially improve inference throughput, but it may also lead to a pronounced degradation in generation quality. For example, LLaDA/2 achieves approximately a $2\times$ speedup over LLaDA; however, its performance decreases on GSM8K, MATH, and HumanEval. This suggests that simply increasing decoding parallelism, in the absence of an explicit correction mechanism, can exacerbate factorization errors and compromise prediction accuracy. This observation validates our motivation that efficient parallel decoding for diffusion language models requires an effective correction mechanism to mitigate factorization errors.

Second, FeF-DLLM provides a stronger accuracy--efficiency trade-off than the baseline, direct speculative decoding methods, and prior approaches for mitigating factorization errors. Averaged over the four benchmarks, FeF-DLLM with step=2 improves the mean accuracy from 46.75 to 51.79, achieving a 5.04 point gain over LLaDA, while boosting the mean decoding speed to $3.86\times$. When using step=4, FeF-DLLM further increases the mean accuracy to 52.06, yielding a 5.31 point improvement over LLaDA, while still providing a $2.33\times$ speedup. Notably, FeF-DLLM also surpasses SSD, a direct speculative decoding baseline, by 3.21 and 3.48 accuracy points under step=2 and step=4, respectively. In addition, compared with factorization-error mitigation methods such as DDOSP and DCD, FeF-DLLM achieves substantially faster inference, showing that resampling-based correction can effectively improve accuracy while preserving high decoding efficiency.

\subsection{Ablation Study}

We conduct four ablation studies to evaluate the contributions of different components in FeF-DLLM, focusing on the effects of training, speculative decoding, draft-model selection and window size of speculative decoding.

\paragraph{Ablation 1: Training Effects}
LLaDA-Instruct can also be directly used for inference with the proposed method without further finetuning its neural network parameters. Meanwhile, to rule out the possibility that the observed improvements are primarily attributable to finetuning rather than to the proposed inference strategy, we also evaluate the trained model using the original inference procedure of LLaDA-Instruct. The results are reported in Table~\ref{tab:ablation_training}.

\begin{table*}[t]
    \centering
    \small
    \setlength{\tabcolsep}{3.8pt}
    \caption{Ablation study on training effects. Best accuracy values within each comparison group are in bold. Mean denotes the average results across all four benchmarks.}
    \label{tab:ablation_training}
    \resizebox{\textwidth}{!}{
    \begin{tabular}{@{}l cc cc cc cc cc@{}}
        \toprule
        \multirow{2}{*}{Methods}
        & \multicolumn{2}{c}{GSM8K}
        & \multicolumn{2}{c}{MATH}
        & \multicolumn{2}{c}{HumanEval}
        & \multicolumn{2}{c}{MBPP}
        & \multicolumn{2}{c}{Mean} \\
        \cmidrule(lr){2-3}
        \cmidrule(lr){4-5}
        \cmidrule(lr){6-7}
        \cmidrule(lr){8-9}
        \cmidrule(lr){10-11}
        & Acc. & Speed & Acc. & Speed & Acc. & Speed & Acc. & Speed & Acc. & Speed \\
        \midrule
        LLaDA
        & 78.60 & $1.00\times$ & 26.60 & $1.00\times$ & 47.60 & $1.00\times$ & 34.20 & $1.00\times$ & 46.75 & $1.00\times$ \\
        LLaDA w/ train
        & 76.19 & $1.00\times$ & 26.38 & $1.00\times$ & 46.34 & $1.00\times$ & 36.20 & $1.00\times$ & 46.28 & $1.00\times$ \\
        \midrule
        FeF-DLLM w/o train (step=2)
        & 77.86 & $3.59\times$ & 35.20 & $3.30\times$ & 46.95 & $3.88\times$ & 40.60 & $4.92\times$ & 50.15 & $3.92\times$ \\
        FeF-DLLM w/o train (step=4)
        & 77.86 & $2.17\times$ & 35.48 & $2.01\times$ & 47.56 & $2.29\times$ & 40.40 & $3.01\times$ & 50.33 & $2.37\times$ \\
        \midrule
        FeF-DLLM (step=2)
        & 79.38 & $3.55\times$ & 36.40 & $3.25\times$ & 48.78 & $3.76\times$ & 42.60 & $4.89\times$ & 51.79 & $3.86\times$ \\
        FeF-DLLM (step=4)
        & \textbf{79.68} & $2.14\times$ & \textbf{36.56} & $1.99\times$ & \textbf{49.39} & $2.21\times$ & \textbf{42.60} & $2.99\times$ & \textbf{52.06} & $2.33\times$ \\
        \bottomrule
    \end{tabular}
    }
\end{table*}

The trained LLaDA model exhibits slight performance drops on GSM8K, MATH, and HumanEval, and only improves on MBPP. In contrast, FeF-DLLM with training consistently outperforms its untrained counterpart across all four benchmarks. These results demonstrate that the gains achieved by our method are driven by the joint effect of finetuning and the proposed inference strategy.

\paragraph{Ablation 2: Effect of Speculative Decoding}

To demonstrate the acceleration effect of speculative decoding, we compare FeF-DLLM with a counterpart that uses the same model but disables speculative decoding. Table~\ref{tab:ablation_speculative} reports the results. Here, FeF-DLLM w/o SD denotes the non-speculative setting. The results show that incorporating speculative decoding substantially improves inference speed while preserving accuracy.

\begin{table*}[t]
    \centering
    \small
    \setlength{\tabcolsep}{3.8pt}
    \caption{Ablation study on the effect of speculative decoding (SD). All models in this table are trained models. Mean denotes the average results across all four benchmarks.}
    \label{tab:ablation_speculative}
    \resizebox{\textwidth}{!}{
    \begin{tabular}{@{}l cc cc cc cc cc@{}}
        \toprule
        \multirow{2}{*}{Methods}
        & \multicolumn{2}{c}{GSM8K}
        & \multicolumn{2}{c}{MATH}
        & \multicolumn{2}{c}{HumanEval}
        & \multicolumn{2}{c}{MBPP}
        & \multicolumn{2}{c}{Mean} \\
        \cmidrule(lr){2-3}
        \cmidrule(lr){4-5}
        \cmidrule(lr){6-7}
        \cmidrule(lr){8-9}
        \cmidrule(lr){10-11}
        & Acc. & Speed & Acc. & Speed & Acc. & Speed & Acc. & Speed & Acc. & Speed \\
        \midrule
        FeF-DLLM w/o SD (step=2)
        & 79.38 & $0.69\times$ & 36.40 & $0.67\times$ & 48.78 & $0.66\times$ & 42.60 & $0.67\times$ & 51.79 & $0.67\times$ \\
        FeF-DLLM w/o SD (step=4)
        & 79.68 & $0.41\times$ & 36.56 & $0.40\times$ & 49.39 & $0.39\times$ & 42.60 & $0.40\times$ & 52.06 & $0.40\times$ \\
        \midrule
        FeF-DLLM (step=2)
        & 79.38 & $\mathbf{3.55\times}$ & 36.40 & $\mathbf{3.25\times}$ & 48.78 & $\mathbf{3.76\times}$ & 42.60 & $\mathbf{4.89\times}$ & 51.79 & $\mathbf{3.86}\times$ \\
        FeF-DLLM (step=4)
        & 79.68 & $\mathbf{2.14\times}$ & 36.56 & $\mathbf{1.99\times}$ & 49.39 & $\mathbf{2.21\times}$ & 42.60 & $\mathbf{2.99\times}$ & 52.06 & $\mathbf{2.33}\times$ \\
        \bottomrule
    \end{tabular}
    }
\end{table*}

\paragraph{Ablation 3: Choice of Draft Model}

In all previous experiments, we use the same model as both the draft model and the verify model by default. Here, we further analyze how the choices of draft models affect final performance. The results are reported in Table~\ref{tab:ablation_draft_verify}. In addition to task accuracy and wall-clock speedup, we also report the average Acceptance, which measures the fraction of draft tokens accepted by the verify model.

\begin{table*}[t]
    \centering
    \small
    \setlength{\tabcolsep}{3.5pt}
    \caption{Ablation study on the choice of draft models. ``A / B'' denotes using A as the draft model and B as the verify model. All results are reported with resampling step $=4$.}
    \label{tab:ablation_draft_verify}
    \begin{tabular}{l *{6}{C{1.55cm}}}
        \toprule
        \multirow{2}{*}{Task}
        & \multicolumn{3}{c}{FeF-DLLM / FeF-DLLM}
        & \multicolumn{3}{c}{FeF-DLLM w/o train / FeF-DLLM} \\
        \cmidrule(lr){2-4} \cmidrule(lr){5-7}
        & Acc. & Speed & Acceptance & Acc. & Speed & Acceptance \\
        \midrule
        GSM8K     & 79.68 & $2.14\times$ & \textbf{63.29} & 79.68 & $2.12\times$ & 62.55 \\
        MATH      & 36.56 & $1.99\times$ & \textbf{58.24} & 36.56 & $1.93\times$ & 56.17 \\
        HumanEval & 49.39 & $2.21\times$ & \textbf{65.99} & 49.39 & $2.17\times$ & 64.52 \\
        MBPP      & 42.60 & $2.99\times$ & \textbf{92.84} & 42.60 & $2.89\times$ & 91.77 \\
        \bottomrule
    \end{tabular}
\end{table*}

The results show that the final accuracy remains unchanged across the two draft-model choices on all four benchmarks. This suggests that, in this setting, the verify model primarily determines the final prediction quality. Meanwhile, using the same model for both drafting and verification leads to consistently higher acceptance rates and slightly better speedup. This observation is consistent with our analysis in Corollary \ref{corollary:acceleration}: when the draft model and the verify model are better aligned, the verify model accepts more draft tokens, which leads to more efficient speculative decoding.

\paragraph{Ablation 4: Speculative Decoding Window Size}

We further study the effect of the speculative decoding window size. Table~\ref{tab:ablation_window_size} reports the results with resampling step $=4$. The results show that increasing the window size substantially improves decoding speed while preserving the same accuracy. Larger window sizes could not be implemented due to device limitations. Therefore, a window size of 16 was selected for the formal experiment.
\begin{table*}[!t]
    \centering
    \small
    \setlength{\tabcolsep}{3.8pt}
    \caption{Ablation study on the speculative decoding window size. All results are reported with resampling step $=4$. Mean denotes the average results across all four benchmarks.}
    \label{tab:ablation_window_size}
    \resizebox{\textwidth}{!}{
    \begin{tabular}{@{}l c cc cc cc cc cc@{}}
        \toprule
        \multirow{2}{*}{Variant}
        & \multirow{2}{*}{Window Size}
        & \multicolumn{2}{c}{GSM8K}
        & \multicolumn{2}{c}{MATH}
        & \multicolumn{2}{c}{HumanEval}
        & \multicolumn{2}{c}{MBPP}
        & \multicolumn{2}{c}{Mean} \\
        \cmidrule(lr){3-4}
        \cmidrule(lr){5-6}
        \cmidrule(lr){7-8}
        \cmidrule(lr){9-10}
        \cmidrule(lr){11-12}
        & & Acc. & Speed & Acc. & Speed & Acc. & Speed & Acc. & Speed & Acc. & Speed \\
        \midrule
        \multirow{3}{*}{FeF-DLLM}
        & 4  & 79.68 & $0.77\times$ & 36.56 & $0.76\times$ & 49.39 & $0.77\times$ & 42.60 & $0.80\times$ & 52.06 & $0.78\times$ \\
        & 8  & 79.68 & $1.38\times$ & 36.56 & $1.34\times$ & 49.39 & $1.40\times$ & 42.60 & $1.57\times$ & 52.06 & $1.42\times$ \\
        & 16 & 79.68 & $\mathbf{2.14\times}$ & 36.56 & $\mathbf{1.99\times}$ & 49.39 & $\mathbf{2.21\times}$ & 42.60 & $\mathbf{2.99\times}$ & 52.06 & $\mathbf{2.33\times}$ \\
        \bottomrule
    \end{tabular}
    }
\end{table*}

\section{Limitation and Conclusion}

In this paper, we proposed FeF-DLLM, a factorization-error-free discrete diffusion language modeling method based on prefix-conditioned clean-token prediction. By replacing independent token-wise prediction with exact prefix-conditioned factorization, FeF-DLLM preserves token dependencies and eliminates factorization error. We further incorporated speculative decoding to accelerate inference while maintaining the correctness of the target distribution. Experiments on mathematical reasoning and code generation benchmarks show that FeF-DLLM improves generation quality and achieves substantial inference speedup.

One limitation of our method is that prefix-conditioned prediction and speculative verification require more computational resources during inference than standard DLLM decoding. Future work will explore more resource-efficient implementations to further reduce this overhead.

\bibliographystyle{plainnat}
\bibliography{references}

@inproceedings{leviathan2023fast,
  title={Fast Inference from Transformers via Speculative Decoding},
  author={Leviathan, Yaniv and Kalman, Matan and Matias, Yossi},
  booktitle={Proceedings of the 40th International Conference on Machine Learning (ICML)},
  pages={19274--19286},
  year={2023},
  publisher={PMLR}
}

@inproceedings{ho2020denoising,
  title={Denoising Diffusion Probabilistic Models},
  author={Ho, Jonathan and Jain, Ajay and Abbeel, Pieter},
  booktitle={Advances in Neural Information Processing Systems (NeurIPS)},
  volume={33},
  pages={6840--6851},
  year={2020}
}

@inproceedings{austin2021structured,
  title={Structured Denoising Diffusion Models in Discrete State-spaces},
  author={Austin, Jacob and Johnson, Daniel D. and Ho, Jonathan and Tarlow, Daniel and Van Den Berg, Rianne},
  booktitle={Advances in Neural Information Processing Systems (NeurIPS)},
  volume={34},
  pages={17981--17993},
  year={2021}
}

@inproceedings{lou2024discrete,
  title={Discrete Diffusion Language Modeling by Estimating the Ratios of the Data Distribution},
  author={Lou, Aaron and Meng, Chenlin and Ermon, Stefano},
  booktitle={International Conference on Learning Representations (ICLR)},
  year={2024}
}

@inproceedings{sahoo2024simple,
  title={Simple and Effective Masked Diffusion Language Models},
  author={Sahoo, Subham Sekhar and Arriola, Marianne and Schiff, Yair and Gokaslan, Aaron and Marroquin, Edgar and Chiu, Justin T. and Rush, Alexander and Kuleshov, Volodymyr},
  booktitle={Advances in Neural Information Processing Systems (NeurIPS)},
  volume={37},
  year={2024}
}

@article{nie2025llada,
  title={{LLaDA}: Large Language Diffusion Models},
  author={Nie, Shen and Zhu, Fengqi and You, Zebin and Zhang, Xiaolu and Ou, Jingyang and Hu, Jun and Zhou, Jun and others},
  journal={arXiv preprint arXiv:2502.09992},
  year={2025}
}

@article{cobbe2021gsm8k,
  title={Training Verifiers to Solve Math Word Problems},
  author={Cobbe, Karl and Kosaraju, Vineet and Bavarian, Mohammad and Chen, Mark and Jun, Heewoo and Kaiser, {\L}ukasz and Plappert, Matthias and Tworek, Jerry and Hilton, Jacob and Nakano, Reiichiro and others},
  journal={arXiv preprint arXiv:2110.14168},
  year={2021}
}

@article{hendrycks2021math,
  title={Measuring Mathematical Problem Solving with the {MATH} Dataset},
  author={Hendrycks, Dan and Burns, Collin and Kadavath, Saurav and Arora, Akul and Basart, Steven and Tang, Eric and Song, Dawn and Steinhardt, Jacob},
  journal={arXiv preprint arXiv:2103.03874},
  year={2021}
}

@article{chen2021humaneval,
  title={Evaluating Large Language Models Trained on Code},
  author={Chen, Mark and Tworek, Jerry and Jun, Heewoo and Yuan, Qiming and Pinto, Henrique P. de Oliveira and Kaplan, Jared and Edwards, Harri and Burda, Yuri and Joseph, Nicholas and Brockman, Greg and others},
  journal={arXiv preprint arXiv:2107.03374},
  year={2021}
}

@article{austin2021mbpp,
  title={Program Synthesis with Large Language Models},
  author={Austin, Jacob and Odena, Augustus and Nye, Maxwell and Bosma, Maarten and Michalewski, Henryk and Dohan, David and Jiang, Ellen and Cai, Carrie and Terry, Michael and Le, Quoc and others},
  journal={arXiv preprint arXiv:2108.07732},
  year={2021}
}

@article{gao2025self,
  title={Self Speculative Decoding for Diffusion Large Language Models},
  author={Gao, Yifeng and Ji, Ziang and Wang, Yuxuan and Qi, Biqing and Xu, Hanlin and Zhang, Linfeng},
  journal={arXiv preprint arXiv:2510.04147},
  year={2025}
}

@inproceedings{lipmanflow,
  title={Flow Matching for Generative Modeling},
  author={Lipman, Yaron and Chen, Ricky TQ and Ben-Hamu, Heli and Nickel, Maximilian and Le, Matthew},
  booktitle={The Eleventh International Conference on Learning Representations},
  year={2022}
}

@inproceedings{songscore,
  title={Score-Based Generative Modeling through Stochastic Differential Equations},
  author={Song, Yang and Sohl-Dickstein, Jascha and Kingma, Diederik P. and Kumar, Abhishek and Ermon, Stefano and Poole, Ben},
  booktitle={International Conference on Learning Representations},
  year={2020}
}

@inproceedings{liudiscrete,
  title={Discrete Copula Diffusion},
  author={Liu, Anji and Broadrick, Oliver and Niepert, Mathias and Van den Broeck, Guy},
  booktitle={The Thirteenth International Conference on Learning Representations},
  year={2024}
}

@article{campbell2025self,
  title={Self-speculative Masked Diffusions},
  author={Campbell, Andrew and De Bortoli, Valentin and Shi, Jiaxin and Doucet, Arnaud},
  journal={arXiv preprint arXiv:2510.03929},
  year={2025}
}

@inproceedings{yooredi,
  title={{ReDi}: Rectified Discrete Flow},
  author={Yoo, Jaehoon and Kim, Wonjung and Hong, Seunghoon},
  booktitle={The Thirty-ninth Annual Conference on Neural Information Processing Systems},
  year={2025}
}

@article{cheng2025deerdraftdiffusionverify,
  title={{DEER}: Draft with Diffusion, Verify with Autoregressive Models},
  author={Cheng, Zicong and Yang, Guo-Wei and Li, Jia and Deng, Zhijie and Guo, Meng-Hao and Hu, Shi-Min},
  journal={arXiv preprint arXiv:2512.15176},
  year={2025}
}

@article{lavenant2025error,
  title={Error Bounds and Optimal Schedules for Masked Diffusions with Factorized Approximations},
  author={Lavenant, Hugo and Zanella, Giacomo},
  journal={arXiv preprint arXiv:2510.25544},
  year={2025}
}

@article{ho2022video,
  title={Video Diffusion Models},
  author={Ho, Jonathan and Salimans, Tim and Gritsenko, Alexey and Chan, William and Norouzi, Mohammad and Fleet, David J.},
  journal={Advances in Neural Information Processing Systems (NeurIPS)},
  volume={35},
  pages={8633--8646},
  year={2022}
}

@inproceedings{bar2024lumiere,
  title={{Lumiere}: A Space-Time Diffusion Model for Video Generation},
  author={Bar-Tal, Omer and Chefer, Hila and Tov, Omer and Herrmann, Charles and Paiss, Roni and Zada, Shiran and Ephrat, Ariel and Hur, Junhwa and Liu, Guanghui and Raj, Amit and others},
  booktitle={SIGGRAPH Asia 2024 Conference Papers},
  pages={1--11},
  year={2024}
}

@inproceedings{peebles2023scalable,
  title={Scalable Diffusion Models with Transformers},
  author={Peebles, William and Xie, Saining},
  booktitle={Proceedings of the IEEE/CVF International Conference on Computer Vision},
  pages={4195--4205},
  year={2023}
}

@article{bie2025llada20scalingdiffusionlanguage,
  title={{LLaDA2.0}: Scaling Up Diffusion Language Models to 100B},
  author={Bie, Tiwei and Cao, Maosong and Chen, Kun and Du, Lun and Gong, Mingliang and Gong, Zhuochen and Gu, Yanmei and Hu, Jiaqi and Huang, Zenan and Lan, Zhenzhong and Li, Chengxi and Li, Chongxuan and Li, Jianguo and Li, Zehuan and Liu, Huabin and Liu, Lin and Lu, Guoshan and Lu, Xiaocheng and Ma, Yuxin and Tan, Jianfeng and Wei, Lanning and Wen, Ji-Rong and Xing, Yipeng and Zhang, Xiaolu and Zhao, Junbo and Zheng, Da and Zhou, Jun and Zhou, Junlin and Zhou, Zhanchao and Zhu, Liwang and Zhuang, Yihong},
  journal={arXiv preprint arXiv:2512.15745},
  year={2025}
}

@misc{2023opencompass,
  title={{OpenCompass}: A Universal Evaluation Platform for Foundation Models},
  author={OpenCompass Contributors},
  howpublished={\url{https://github.com/open-compass/opencompass}},
  year={2023}
}

@inproceedings{songdenoising,
  title={Denoising Diffusion Implicit Models},
  author={Song, Jiaming and Meng, Chenlin and Ermon, Stefano},
  booktitle={International Conference on Learning Representations},
  year={2020}
}

@article{ye2025dream,
  title={{Dream} 7B: Diffusion Large Language Models},
  author={Ye, Jiacheng and Xie, Zhihui and Zheng, Lin and Gao, Jiahui and Wu, Zirui and Jiang, Xin and Li, Zhenguo and Kong, Lingpeng},
  journal={arXiv preprint arXiv:2508.15487},
  year={2025}
}

@article{gat2024discrete,
  title={Discrete Flow Matching},
  author={Gat, Itai and Remez, Tal and Shaul, Neta and Kreuk, Felix and Chen, Ricky TQ and Synnaeve, Gabriel and Adi, Yossi and Lipman, Yaron},
  journal={Advances in Neural Information Processing Systems (NeurIPS)},
  volume={37},
  pages={133345--133385},
  year={2024}
}

@inproceedings{rombach2022high,
  title={High-Resolution Image Synthesis with Latent Diffusion Models},
  author={Rombach, Robin and Blattmann, Andreas and Lorenz, Dominik and Esser, Patrick and Ommer, Bj{\"o}rn},
  booktitle={Proceedings of the IEEE/CVF Conference on Computer Vision and Pattern Recognition},
  pages={10684--10695},
  year={2022}
}

@inproceedings{ma2024sit,
  title={{SiT}: Exploring Flow and Diffusion-Based Generative Models with Scalable Interpolant Transformers},
  author={Ma, Nanye and Goldstein, Mark and Albergo, Michael S. and Boffi, Nicholas M. and Vanden-Eijnden, Eric and Xie, Saining},
  booktitle={European Conference on Computer Vision},
  pages={23--40},
  year={2024},
  organization={Springer}
}

@inproceedings{wu2024seesr,
  title={{SeeSR}: Towards Semantics-Aware Real-World Image Super-Resolution},
  author={Wu, Rongyuan and Yang, Tao and Sun, Lingchen and Zhang, Zhengqiang and Li, Shuai and Zhang, Lei},
  booktitle={Proceedings of the IEEE/CVF Conference on Computer Vision and Pattern Recognition},
  pages={25456--25467},
  year={2024}
}

@article{yim2024diffusion,
  title={Diffusion Models in Protein Structure and Docking},
  author={Yim, Jason and St{\"a}rk, Hannes and Corso, Gabriele and Jing, Bowen and Barzilay, Regina and Jaakkola, Tommi S.},
  journal={Wiley Interdisciplinary Reviews: Computational Molecular Science},
  volume={14},
  number={2},
  pages={e1711},
  year={2024},
  publisher={Wiley Online Library}
}

@article{wang2025diffusion,
  title={Diffusion Models for 3D Generation: A Survey},
  author={Wang, Chen and Peng, Hao-Yang and Liu, Ying-Tian and Gu, Jiatao and Hu, Shi-Min},
  journal={Computational Visual Media},
  volume={11},
  number={1},
  pages={1--28},
  year={2025},
  publisher={TUP}
}

@article{chen2023accelerating,
  title={Accelerating Large Language Model Decoding with Speculative Sampling},
  author={Chen, Charlie and Borgeaud, Sebastian and Irving, Geoffrey and Lespiau, Jean-Baptiste and Sifre, Laurent and Jumper, John},
  journal={arXiv preprint arXiv:2302.01318},
  year={2023}
}

\appendix

\section{Implementation Details of Compared Methods}
\label{sec:compared_methods}

We compare against three representative inference-time baselines, namely SSD \citep{gao2025self}, DCD \citep{liudiscrete}, and DDOSP \citep{lavenant2025error}. For all comparing methods, we use the same LLaDA-Instruct checkpoint, prompt formatting, maximum generation length, and OpenCompass evaluation pipeline as in the main experiments. Unless otherwise stated, we keep the denoising step budget, block partition, temperature, classifier-free guidance scale, and re-masking rule identical to the LLaDA baseline, so that the differences come only from the inference strategy itself.

\paragraph{SSD.}
We implement SSD as a training-free decoding wrapper around the same LLaDA checkpoint. At each denoising step, the model first predicts all masked positions in parallel and uses the top-1 token at each masked position as the self-draft token. Candidate positions are selected within the current semi-autoregressive decoding block according to model confidence, and a greedy linear verification tree is constructed. All verification nodes are then packed into one batch and verified by the same LLaDA model. Therefore, SSD does not introduce an auxiliary draft model or any additional training. In the experiments, the same checkpoint is used as both the drafter and the verifier.

\paragraph{DCD.}
The original DCD formulation requires a diffusion model and an autoregressive copula model that share exactly the same tokenizer and token-to-id mapping. Because standard causal language models such as GPT-2 are not tokenizer-compatible with LLaDA, we implement a causalized DCD-like baseline by reusing the same LLaDA checkpoint with a causal attention bias as the copula model. At each denoising step, we compute diffusion logits $l_{\mathrm{diff}}$ and causalized copula logits $l_{\mathrm{cop}}$, and form the proposal logits as
\[
    l_{\mathrm{prop}} = l_{\mathrm{cop}} + \alpha l_{\mathrm{diff}},
\]
where $\alpha$ controls the strength of the diffusion proposal and is set to 1.0 by default. For efficiency, we use the one-pass causal proposal mode in the main experiments, namely one causal forward pass per denoising step, while the sequential mode is reserved for additional analysis because it is substantially slower. Since LLaDA does not expose the same score/transition pair as the original Copula-Diffusion implementation, this baseline should be regarded as a DCD-style approximation rather than an exact reimplementation of the original method.

\paragraph{DDOSP.}
DDOSP does not train a new model; instead, it replaces the handcrafted uniform unmasking schedule with a data-driven non-uniform schedule estimated from the denoiser. Specifically, we use training-set response sequences to estimate the information profile $f(i)$ under different numbers of revealed tokens, where the expectation is approximated by Monte Carlo masking and batched denoiser forward passes. The estimated profile is then smoothed and further corrected to be monotone before computing the incremental information gains $\Delta f(i)$. Based on these gains, we construct a $K$-step decoding schedule, cache the resulting schedule file, and reuse it during inference. During sampling, DDOSP only changes how many masked positions are released at each denoising round; the denoiser, token proposal rule, blockwise generation order, temperature, and re-masking strategy remain the same as in the LLaDA baseline. When no precomputed schedule is available, the implementation falls back to the uniform schedule.

\newpage

\section{Implementation Details of Training}

This section supplements the training details not explicitly reported in Section~\ref{sec:exp_setup}. In addition to the hardware and optimization settings described there, we report the released supplementary training runs and their usage in evaluation.

For the supplementary training setup, all runs were executed on 8 NVIDIA A100 80GB GPUs. The math training run used 18k examples for 1 epoch and took approximately 17.6 hours, corresponding to about 140.8 GPU-hours; the resulting model was used for evaluation on GSM8K and MATH. The code training run used 54k examples for 1 epoch and took approximately 52.8 hours, corresponding to about 422.4 GPU-hours; the resulting model was used for evaluation on HumanEval and MBPP. Summing these two runs gives an approximate total of 563.2 GPU-hours for the released supplementary training runs.

Additional implementation details are provided in Appendix~\ref{sec:compared_methods} and Appendix~\ref{sec:supp_algorithm}.

\newpage

\section{Algorithm}
\label{sec:supp_algorithm}

This section provides the detailed algorithms for FeF-DLLM.
Algorithm~\ref{alg:position_conditioned_training} summarizes the
position-conditioned training procedure. Algorithm~\ref{alg:fef_dllm_full_inference}
describes the full-sequence inference procedure, which combines speculative
verification with low-confidence re-corruption.

Before presenting the inference algorithm, we define the residual distribution
used when a draft token is rejected. For a target law $\pi_\theta^{j_m}$ and a
draft law $\rho_\phi^{j_m}$ at position $j_m$, define
\begin{equation}
    \pi_\theta^{\prime j_m}(x)
    =
    \frac{
        \left[\pi_\theta^{j_m}(x)-\rho_\phi^{j_m}(x)\right]_+
    }{
        1-\sum_{z\in\mathcal{V}}
        \min\{\pi_\theta^{j_m}(z),\rho_\phi^{j_m}(z)\}
    },
    \label{eq:residual_distribution_alg}
\end{equation}
where $[u]_+=\max\{u,0\}$. If the denominator is zero, the rejection event has
probability zero, and the residual distribution is never sampled.

\begin{algorithm}[t]
\caption{Position-Conditioned Training for FeF-DLLM}
\label{alg:position_conditioned_training}
\begin{algorithmic}[1]
\Require Predictor $p_\theta$, data distribution $q(X_0)$, response position set $J_{\mathrm{resp}}$
\Ensure Trained predictor $p_\theta$
\Repeat
    \State Sample clean data $X_0\sim q(X_0)$.
    \State Sample diffusion time $t$.
    \State Sample corrupted data $X_t\sim q(X_t\mid X_0)$.
    \State Define the ordered prediction positions:
    \Statex
    \[
        J_t
        =
        \{i\in J_{\mathrm{resp}}: X_t^i \neq X_0^i\}
        =
        \{j_1,\dots,j_{L_t}\},
        \qquad
        j_1<\cdots<j_{L_t},
        \qquad
        L_t=|J_t|.
    \]
    \State Initialize loss $\mathcal{L}\leftarrow 0$.
    \For{$m=1$ to $L_t$}
        \State Construct the position-conditioned input $X^{(j_m)}$:
        \Statex
        \[
        X^{(j_m),i}
        =
        \begin{cases}
            X_0^i, & i<j_m, \\
            X_t^i, & i\ge j_m .
        \end{cases}
        \]
        \State Compute the token-level cross-entropy loss:
        \Statex
        \[
            \mathrm{CE}_{j_m}
            =
            -
            \log
            p_\theta
            \left(
                X_0^{j_m}
                \mid
                X^{(j_m)},t
            \right).
        \]
        \State Accumulate loss:
        \Statex
        \[
            \mathcal{L}\leftarrow \mathcal{L}+\mathrm{CE}_{j_m}.
        \]
    \EndFor
    \State Compute $\nabla_\theta \mathcal{L}$ and update $\theta$ with the optimizer.
\Until{converged}
\State \Return $p_\theta$
\end{algorithmic}
\end{algorithm}

Algorithm~\ref{alg:position_conditioned_training} trains the model to predict
each clean token from a position-conditioned input. For position $j_m$, tokens
before $j_m$ are replaced by their clean values from $X_0$, while position
$j_m$ and the suffix remain in the corrupted state $X_t$. This construction
matches the verification-time input used by the target model during
left-to-right speculative verification, thereby aligning training with
factorization-error-free inference.

\begin{algorithm}[ht]
\caption{Full-Sequence FeF-DLLM Decoding}
\label{alg:fef_dllm_full_inference}
\begin{algorithmic}[1]
\Require Initial corrupted sequence $X_t$, prediction positions $J_{\mathrm{pred}}$, draft model $M_\rho$, target model $M_\pi$, window size $k$, number of steps $N_{\mathrm{step}}$, resampling budget $n_{\mathrm{rm}}$
\Ensure Final generated sequence $\hat X_0$
\State Initialize $X\leftarrow X_t$.
\For{$s=0$ to $N_{\mathrm{step}}-1$}
    \State Initialize $\mathcal{C}^{(s)}\leftarrow \emptyset$.
    \While{there exists an unresolved position in $J_{\mathrm{pred}}$}
        \State Let $J_t=\{j_1,\dots,j_{L_t}\}\subseteq J_{\mathrm{pred}}$, with $j_1<\cdots<j_{L_t}$, be the unresolved positions.
        \State Set $b\leftarrow \min\{k,L_t\}$.
        \State Compute draft laws $\rho_\phi^{j_m}(\cdot)=\rho_\phi(X_0^{j_m}=\cdot\mid X,t)$ for $m=1,\dots,b$.
        \State Sample draft tokens $\tilde X_0^{j_m}\sim \rho_\phi^{j_m}(\cdot)$ for $m=1,\dots,b$.
        \State Construct verification inputs $X^{(j_m)}$ for $m=1,\dots,b$,
        \State \hspace{\algorithmicindent}each input uses the verified prefix and the draft prefix
        $\tilde X_0^{j_1},\dots,\tilde X_0^{j_{m-1}}$.
        \State Compute target laws $\pi_\theta^{j_m}(\cdot)=p_\theta(X_0^{j_m}=\cdot\mid X^{(j_m)},t)$ for $m=1,\dots,b$.
        \For{$m=1$ to $b$}
            \State Sample $u_m\sim\mathrm{Uniform}(0,1)$.
            \If{$u_m\le \min\{1,\pi_\theta^{j_m}(\tilde X_0^{j_m})/\rho_\phi^{j_m}(\tilde X_0^{j_m})\}$}
                \State $X^{j_m}\leftarrow \tilde X_0^{j_m}$,\quad $c_{j_m}\leftarrow \pi_\theta^{j_m}(\tilde X_0^{j_m})$.
                \State $\mathcal{C}^{(s)}\leftarrow \mathcal{C}^{(s)}\cup\{(j_m,c_{j_m})\}$.
            \Else
                \State Sample $Z^{j_m}\sim \pi_\theta^{\prime j_m}(\cdot)$ according to Eq.~\ref{eq:residual_distribution_alg}.
                \State $X^{j_m}\leftarrow Z^{j_m}$,\quad $c_{j_m}\leftarrow \pi_\theta^{j_m}(Z^{j_m})$.
                \State Discard the remaining draft tokens and restart from the updated $X$.
                \State $\mathcal{C}^{(s)}\leftarrow \mathcal{C}^{(s)}\cup\{(j_m,c_{j_m})\}$.
                \State \textbf{break}
            \EndIf
        \EndFor
    \EndWhile
    \State $\hat X_0^{(s)}\leftarrow X$.
    \If{$s<N_{\mathrm{step}}-1$}
        \State Select $R^{(s+1)}=\operatorname{LowConf}_{n_{\mathrm{rm}}}(\mathcal{C}^{(s)})$.
        \For{each $i\in R^{(s+1)}$}
            \State Re-corrupt $X^i\sim q(X_t^i\mid \hat X_0^{(s),i})$.
        \EndFor
    \EndIf
\EndFor
\State \Return $\hat X_0=\hat X_0^{(N_{\mathrm{step}}-1)}$
\end{algorithmic}
\end{algorithm}
Algorithm~\ref{alg:fef_dllm_full_inference} describes the full-sequence inference procedure of FeF-DLLM. At each generation step, the method repeatedly
applies speculative decoding over the unresolved prediction positions. The draft
model $M_\rho$ proposes up to $k$ candidate clean tokens in the current
speculative window, and the target model $M_\pi$ verifies these candidates in a
fixed left-to-right order using position-conditioned inputs. Accepted or
corrected tokens are written back to the sequence and serve as clean prefix
context for subsequent positions. If all candidates in the window are accepted,
the whole window is committed; if one candidate is rejected, a corrected token is
sampled from the residual distribution, the remaining draft tokens in the window
are discarded, and a new speculative window starts from the updated sequence.
After all prediction positions are resolved, the method selects low-confidence
positions and re-corrupts them through the D3PM forward kernel for the next
generation step. This procedure preserves the accept--reject correction of
speculative decoding while allowing low-confidence positions to be refined
through repeated generation and re-corruption steps.

\newpage

\section{Theoretical Proofs}
\subsection{Proof of Lemma \ref{lemma:predict}}
\begin{proof}
For simplicity, we omit the timestep index in the probability notation. By the
chain rule and the position-wise corruption assumption, we have
\[
\begin{aligned}
p(X_0^i\mid X_t,X_0^{<i})
&=
\frac{
p(X_0^i\mid X_t)p(X_0^{<i}\mid X_t,X_0^i)
}{
p(X_0^{<i}\mid X_t)
} \\
&=
\frac{
p(X_0^i\mid X_t)p(X_t\mid X_0^{<i},X_0^i)p(X_0^{<i}\mid X_0^i)
}{
p(X_0^{<i}\mid X_t)p(X_t\mid X_0^i)
} \\
&=
\frac{
p(X_0^i\mid X_t)
p(X_t^{<i}\mid X_0^{<i})
p(X_t^i\mid X_0^i)
p(X_t^{>i}\mid X_0^i,X_0^{<i})
p(X_0^{<i}\mid X_0^i)
}{
p(X_0^{<i}\mid X_t)
p(X_t^i\mid X_0^i)
p(X_t^{-i}\mid X_0^i)
} \\
&=
\frac{
p(X_0^i\mid X_t)
p(X_t^{<i}\mid X_0^{<i})
p(X_t^{>i}\mid X_0^i,X_0^{<i})
p(X_0^{<i}\mid X_0^i)
}{
p(X_0^{<i}\mid X_t)
p(X_t^{-i}\mid X_0^i)
} \\
&\propto
\frac{
p(X_0^i\mid X_t)
p(X_t^{>i}\mid X_0^i,X_0^{<i})
p(X_0^{<i}\mid X_0^i)
}{
p(X_t^{-i}\mid X_0^i)
} \\
&\propto
p(X_0^i\mid X_t)
\frac{
p(X_0^i\mid X_t^{>i},X_0^{<i})
}{
p(X_0^i\mid X_t^{-i})
}.
\end{aligned}
\]
Moreover, since the forward corruption process factorizes across positions,
$X_t^{<i}$ is generated only from $X_0^{<i}$. Hence, once $X_0^{<i}$ is
conditioned on, $X_t^{<i}$ provides no additional information about $X_0^i$,
which gives
\[
p(X_0^i\mid X_t,X_0^{<i})
=
p(X_0^i\mid X_t^{\ge i},X_0^{<i}).
\]
The first equality follows from the chain rule of conditional probabilities. The
second equality expands $p(X_0^{<i}\mid X_t,X_0^i)$ by Bayes' rule. The third
equality decomposes the corrupted sequence into prefix, current position, and
suffix under the position-wise forward corruption process, with the suffix term
understood as the marginal likelihood after integrating out $X_0^{>i}$. The
fourth equality cancels the common factor $p(X_t^i\mid X_0^i)$. The first
proportionality absorbs all terms independent of $X_0^i$ into the normalization
constant. The last proportionality rewrites the remaining likelihood ratio in
terms of $p(X_0^i\mid X_t^{>i},X_0^{<i})$ and
$p(X_0^i\mid X_t^{-i})$. The proportionality is over $X_0^i$, with
normalization over the vocabulary.
\end{proof}

\subsection{Proof of Theorem \ref{theorem:distribution}}

\begin{proof}
Fix the corrupted state $X_t$ and the ordered update set
\[
    J_t=\{j_1,\dots,j_{L_t}\},
    \qquad
    j_1<\cdots<j_{L_t}.
\]
For each $j_m\in J_t$, the target model is instantiated as the oracle
next-position law
\[
    \pi_\theta^{j_m}(x)
    =
    p^\star
    \left(
        X_0^{j_m}=x
        \mid
        X_t^{\ge j_m},
        X_0^{<j_m},
        t
    \right),
    \qquad m=1,\dots,L_t .
\]
Let $\rho_\phi^{j_m}$ denote the corresponding draft law.

We first verify the one-step correction. Condition on the already committed
prefix before position $j_m$. Then both $\pi_\theta^{j_m}$ and
$\rho_\phi^{j_m}$ are fixed categorical distributions on $\mathcal{V}$.
A proposal $\tilde X_0^{j_m}\sim \rho_\phi^{j_m}$ is accepted with probability
\[
    a(\tilde X_0^{j_m})
    =
    \min
    \left\{
        1,
        \frac{\pi_\theta^{j_m}(\tilde X_0^{j_m})}
        {\rho_\phi^{j_m}(\tilde X_0^{j_m})}
    \right\}.
\]
We use the convention
\[
    \rho_\phi^{j_m}(x)
    \min\left\{1,\frac{\pi_\theta^{j_m}(x)}
    {\rho_\phi^{j_m}(x)}\right\}
    =
    \min\{\pi_\theta^{j_m}(x),\rho_\phi^{j_m}(x)\},
\]
which also covers the case $\rho_\phi^{j_m}(x)=0$. If the proposal is rejected,
the corrected token is sampled from
\[
    \pi_\theta^{\prime j_m}(x)
    =
    \frac{
        \left[\pi_\theta^{j_m}(x)-\rho_\phi^{j_m}(x)\right]_+
    }{
        1-\sum_{z\in\mathcal{V}}
        \min\{\pi_\theta^{j_m}(z),\rho_\phi^{j_m}(z)\}
    },
\]
where $[u]_+=\max\{u,0\}$. If the denominator is zero, rejection occurs with
probability zero, and the residual law is never used.

For any $x\in\mathcal{V}$, the probability that $x$ is committed at position
$j_m$ is
\[
\begin{aligned}
    \Pr(\hat X_0^{j_m}=x)
    &=
    \rho_\phi^{j_m}(x)
    \min
    \left\{
        1,
        \frac{\pi_\theta^{j_m}(x)}
        {\rho_\phi^{j_m}(x)}
    \right\}
    +
    \left(
        1-\sum_{z\in\mathcal{V}}
        \min\{\pi_\theta^{j_m}(z),\rho_\phi^{j_m}(z)\}
    \right)
    \pi_\theta^{\prime j_m}(x)
    \\
    &=
    \min\{\pi_\theta^{j_m}(x),\rho_\phi^{j_m}(x)\}
    +
    \pi_\theta^{j_m}(x)
    -
    \min\{\pi_\theta^{j_m}(x),\rho_\phi^{j_m}(x)\}
    \\
    &=
    \pi_\theta^{j_m}(x).
\end{aligned}
\]
Thus, given the committed prefix, the token committed by the speculative
accept--reject step has law $\pi_\theta^{j_m}$.

Applying this argument sequentially along
$j_1<\cdots<j_{L_t}$, for any
$x_{J_t}=(x^{j_1},\dots,x^{j_{L_t}})\in\mathcal{V}^{L_t}$,
\[
\begin{aligned}
&
\Pr
\left(
    \hat X_{0,J_t}
    =
    x_{J_t}
    \mid
    X_t,t
\right)
\\
&\quad =
\prod_{m=1}^{L_t}
\Pr
\left(
    \hat X_0^{j_m}=x^{j_m}
    \mid
    X_t^{\ge j_m},
    x_0^{<j_m},
    t
\right)
\\
&\quad =
\prod_{m=1}^{L_t}
\pi_\theta^{j_m}(x^{j_m})
\\
&\quad =
\prod_{m=1}^{L_t}
p^\star
\left(
    X_0^{j_m}=x^{j_m}
    \mid
    X_t^{\ge j_m},
    x_0^{<j_m},
    t
\right)
\\
&\quad =
p^\star(x_{J_t}\mid X_t,t).
\end{aligned}
\]
The last equality follows from the left-to-right factorization of the oracle
joint law over the ordered set $J_t$. Hence the generated sequence has the
desired oracle joint law.
\end{proof}

\subsection{Proof of Corollary \ref{corollary:resample}}

\begin{proof}
Fix a resampling pass $s\ge 1$. Let
\[
    R^{(s)}
    =
    \{r_1^{(s)},\dots,r_{L_s}^{(s)}\},
    \qquad
    r_1^{(s)}<\cdots<r_{L_s}^{(s)},
\]
be the positions selected for resampling. We condition on the previous sequence
$\hat X_0^{(s-1)}$ and on the selected set $R^{(s)}$. Under this conditioning,
all positions outside $R^{(s)}$ are fixed, and the resampling pass only updates
positions in $R^{(s)}$.

By the assumption of the corollary, the target model on $R^{(s)}$ is given by
the oracle next-position law. Therefore, the same one-step accept--reject
argument used in the proof of Theorem~1 implies that each committed token has
the corresponding oracle law given the already committed prefix and the fixed
outside positions. Applying the chain rule along the order
\[
    r_1^{(s)}<\cdots<r_{L_s}^{(s)},
\]
we obtain, for any
$x_{R^{(s)}}\in\mathcal{V}^{L_s}$,
\[
\begin{aligned}
&
\Pr
\left(
    \hat X_{0,R^{(s)}}^{(s)}
    =
    x_{R^{(s)}}
    \mid
    \hat X_0^{(s-1)},
    R^{(s)},
    t
\right)
\\
&\quad =
p^\star
\left(
    x_{R^{(s)}}
    \mid
    \hat X_0^{(s-1)},
    R^{(s)},
    t
\right).
\end{aligned}
\]
Thus, conditional on the resampling history and the selected resampling set,
the regeneration step has the oracle joint law on the updated positions.
\end{proof}

\subsection{Proof of Theorem \ref{theorem:committed length}}

\begin{proof}
In one speculative round, candidates are verified from left to right until the
first rejection or until all $k$ candidates are accepted. Let $C_k$ denote the
number of positions committed in this round. Since a rejected position is also
committed after correction, the event $C_k>\ell$ is equivalent to accepting the
first $\ell$ draft tokens.

Under the independent-acceptance approximation, each draft token is accepted
with probability $\alpha$. Hence, for $\ell=0,\dots,k-1$,
\[
    \Pr(C_k>\ell)=\alpha^\ell .
\]
By the tail-sum formula for nonnegative integer-valued random variables,
\[
\begin{aligned}
    \mathbb{E}[C_k]
    &=
    \sum_{\ell=0}^{k-1}\Pr(C_k>\ell)
    \\
    &=
    \sum_{\ell=0}^{k-1}\alpha^\ell .
\end{aligned}
\]
Evaluating the geometric series gives
\[
    \mathbb{E}[C_k]
    =
    \begin{cases}
        \dfrac{1-\alpha^{k}}{1-\alpha}, & 0 \le \alpha < 1, \\[8pt]
        k, & \alpha = 1.
    \end{cases}
\]
\end{proof}

\subsection{Proof of Corollary \ref{corollary:acceleration}}

\begin{proof}
A speculative round consists of one draft proposal pass and one target
verification pass, with wall-clock costs $c_\rho$ and $c_\pi$, respectively.
Thus, its expected cost is $c_\rho+c_\pi$, and by Theorem~2 it commits
$\mathbb{E}[C_k]$ positions in expectation.

The prefix-conditioned sequential target baseline commits one position per
target pass, with cost $c_\pi$. Therefore, its expected cost for committing
$\mathbb{E}[C_k]$ positions is $\mathbb{E}[C_k]c_\pi$. The idealized
acceleration ratio is consequently
\[
    S
    =
    \frac{\mathbb{E}[C_k]\,c_\pi}{c_\rho+c_\pi}.
\]
if $c_\rho \approx c_\pi$, then
\[
    S
    =
    \frac{\mathbb{E}[C_k]}{2}.
\]
Substituting the expression for $\mathbb{E}[C_k]$ from Theorem~2 yields
\[
    S
    =
    \begin{cases}
        \dfrac{1-\alpha^{k}}{2(1-\alpha)}, & 0 \le \alpha < 1, \\[8pt]
        \dfrac{k}{2}, & \alpha = 1.
    \end{cases}
\]
\end{proof}


\end{document}